\documentclass{article}

\PassOptionsToPackage{numbers, compress}{natbib}
\usepackage[final]{nips_2017}

\usepackage[utf8]{inputenc} 
\usepackage[T1]{fontenc}    
\usepackage{hyperref}       
\usepackage{url}            
\usepackage{booktabs}       
\usepackage{amsfonts}       
\usepackage{nicefrac}       
\usepackage{microtype}      
\usepackage{graphicx}

\usepackage{amsmath}

\usepackage{algorithm, algorithmic}

\usepackage{amsthm}
\usepackage{amssymb}

\newtheorem{theorem}{Theorem}
\newtheorem{lemma}[theorem]{Lemma}

\newtheorem{definition}[theorem]{Definition}

\newcommand{\SCer}[2]{\operatorname{ER}(#1, #2)}

\newcommand{\SCes}{\mathcal{P}}
\newcommand{\SCfi}{\mathcal{Q}}
\newcommand{\SCver}{\mathcal{V}}

\newcommand{\SCtaken}[1]{A_{#1}}
\newcommand{\SCnei}[2]{N_{#1}(#2)}

\newcommand{\SCclu}[1]{C_{#1}}
\newcommand{\SCtclu}[1]{D_{#1}}

\newcommand{\SCtnum}{k}
\newcommand{\SCsel}[1]{\alpha_{#1}}
\newcommand{\SClab}[1]{\mu_{\mathcal{D}}(#1)}

\newcommand{\SCgam}[1]{\Gamma(#1)}
\newcommand{\SCc}[1]{\gamma(#1)}
\newcommand{\SCmat}[1]{\psi(#1)}
\newcommand{\SCus}{\Delta_b}
\newcommand{\SCinj}{\Upsilon}
\newcommand{\SCmis}[1]{\mathcal{M}(#1)}

\newcommand{\SCsmt}[2]{\operatorname{HA}(#1,#2)}

\newcommand{\SCcs}{\mathcal{C}}
\newcommand{\SCtcs}{\mathcal{D}}

\newcommand{\jdist}{{\mbox{\sc {dist}}}}
\newcommand{\rgc}{{\mbox{\sc {rgca}}}}

\newcommand{\SCnum}{\ell}
\newcommand{\VV}{V}
\newcommand{\PP}{\mathcal{P}}
\newcommand{\QQ}{\mathcal{Q}}
\newcommand{\CC}{\mathcal{C}}
\newcommand{\DD}{\mathcal{D}}
\newcommand{\SCano}{\Lambda_b}
\newcommand{\SCcen}{\Omega_b}

\newcommand{\scL}{\mathcal{L}}
\newcommand{\field}[1]{\mathbb{#1}}
\newcommand{\scC}{\mathcal{C}}
\newcommand{\E}{\field{E}}
\newcommand{\scO}{\mathcal{O}}

\newcommand{\scG}{\mathcal{G}}
\newcommand{\scH}{\mathcal{H}}

\newcommand{\scP}{\mathcal{P}}
\newcommand{\scS}{\mathcal{S}}

\newcommand{\PTr}{S}

\newcommand{\bnc}{\textsc{saca}}

\title{On Pairwise Clustering with Side Information}

\author{
  Stephen Pasteris  \\
  Department of Computer Science\\
  University College London\\
  London, UK \\
  \texttt{S.Pasteris@cs.ucl.ac.uk} \\
  \And
  Fabio Vitale  \\
  INRIA Lille\\
  Lille, France\\
  \texttt{fabio.vitale@inria.fr}
  \And
  Claudio Gentile \\
  DiSTA\\
  University of Insubria\\
  Varese, Italy\\
  \texttt{claudio.gentile@uninsubria.it}
  \And
  Mark Herbster\\
  Department of Computer Science\\
  University College London\\
  London, UK \\
  \texttt{M.Herbster@cs.ucl.ac.uk} \\
}

\begin{document}

\maketitle

\begin{abstract}
Pairwise clustering, in general, partitions a set of items via a known similarity function.  In our treatment, clustering is modeled as a transductive prediction problem.  Thus rather than beginning with a known similarity function, the function instead is hidden and the learner only receives a random sample consisting of a subset of the pairwise similarities.
An additional set of pairwise side-information may be given to the learner, which then determines the inductive bias of our algorithms. 
We measure performance not based on the recovery of the hidden similarity function, but instead on how well we classify each item.  We give tight bounds on the number of misclassifications.   We provide two algorithms.  The first algorithm \bnc\ is a simple agglomerative clustering algorithm which runs in near linear time, and which serves as a baseline for our analyses.   Whereas the  second algorithm, \rgc, enables the incorporation of side-information which may lead to improved bounds at the cost of a longer running time.
\end{abstract}
\section{Introduction}\label{s:intro}

The aim of clustering is to partition a set of $n$ items into $k$ ``clusters'' based on their  similarity.  A common approach to clustering is to assume that items can be embedded in a metric space, and then to (approximately) minimize an objective function over all possible partitionings based on the metric at hand. A quintessential example is the $k$-means objective.  An alternative is to  assume only the existence of a similarity function between the pairs.  Examples of this approach include spectral~\cite{L13} and $k$-median~\cite{kh79}, as well as correlation clustering~\cite{bbc04}. 

Our approach to clustering is most similar to correlation clustering.
Correlation clustering was introduced in the seminal paper~\cite{bbc04}. In this setting, a complete graph of similarity and dissimilarity item pairs is given. 
The goal is to find a disjoint partition  (``clustering'' ) which minimizes an objective that counts the total number of ``incorrect'' similarity and dissimilarity pairs in the resulting clustering.  A pair of items is incorrect with respect to the clustering if it was given as similar while they appear in distinct clusters, and vice versa.  In~\cite{bbc04} an efficient algorithm with a guaranteed approximation ratio was given for this NP-hard problem. 

Although inspired by these results, our focus is slightly different: we seek to provide efficient algorithms that compute a clustering, as well as to provide predictive performance guarantees for these algorithms.

We treat pairwise clustering as a transductive prediction problem. 
Given a set of unlabeled items, the aim is to predict their class labels.
As input to our algorithms, firstly we have a training set of similarity and dissimilarity item pairs.  
Secondly, we have a set of {\em soft} similarity pairwise constraints -- the {\em side-information} graph.  The side-information graph determines our inductive bias, i.e., our output clustering will tend to (but need not) place each softly constrained item pair into the same cluster.  We give bounds for a batch learning model
where the learner samples uniformly at random a training set of $m$ similarity/dissimilarity pairs from a ground truth clustering. 
Given these $m$ pairs and the side-information graph, the learner then outputs a clustering.   The quality of the resulting clustering is measured by the item {\em misclassification error} which is essentially the number of items in the learner's output clustering that are misclassified as compared to the ground truth. We describe and analyze two novel algorithms for pairwise clustering, and deliver upper bounds on their expected misclassification error which scale to the degree that a clustering exists that reflects the inductive bias induced by the side-information graph at hand. We complement our upper bound with an almost matching lower bound on the prediction complexity of this problem.

The paper is organized as follows. In Section~\ref{s:prel} we review notation as well as formally introduce our learning models. 
In Section~\ref{s:algs}, we present our two clustering algorithms~\rgc{} and~\bnc{}, along with their analyses.  
%
%
Our fastest algorithm is quite efficient, for it requires only a 
linear time in the input size 
up to a sub-logarithmic factor 
to compute a clustering with small error. 
We give concluding remarks in Section~\ref{s:conc}.  Finally, below we provide pointers to a few references in distinct but closely related research areas.

\subsection*{Related work}
The literature most directly related to our work in perspective is the literature on clustering with side information, as well the literature on semi-supervised clustering.  Some of the references in this area include~\cite{bdsy99,rh12}.  Secondarily, our work is also connected to the metric learning task. 
Metric learning is also concerned with recovering a similarity function; however, in this literature the similarity is treated as a real-valued function often identified with a positive semi-definite matrix as opposed to our binary model.  Some relevant references here include~\cite{xnjr-02,m08,cgy12}. What distinguishes our work from the past literature is that we are aimed at constructing clusterings with side information, not just similarity functions, with an associated tight misclassification error analysis.

\section{Preliminaries and Notation}\label{s:prel}
We now introduce our main notation along with basic preliminaries. Given a finite set $\VV = \{1,\ldots,n\}$, a {\em clustering} $\DD$ over $\VV$ is a partition of $\VV$ into a finite number of sets $\DD = \{D_1,\ldots,D_k\}$. Each $D_j$ is called a {\em cluster}. A {\em similarity} graph $G = (\VV,\PP)$ over $\VV$ is an undirected (but not necessarily connected) graph where, for each pairing $(v,w) \in \VV^2$, $v$ and $w$ are {\em similar} if $(v,w) \in \PP$, and {\em dissimilar}, otherwise. Notice that the similarity relationship so defined need not be transitive. We shall interchangeably represent a similarity graph over $\VV$ through a binary $n\times n$ {\em similarity} matrix $Y = [y_{v,w}]_{v,w=1}^{n\times n}$ whose entry $y_{v,w} $ is 1 if items $v$ and $w$ are similar, and $y_{v,w}=0$, otherwise.
A clustering $\DD$ over $\VV$ can be naturally associated with a similarity graph $G = (\VV,\PP_{\DD})$ whose edge set $\PP_{\mathcal{D}}$ is defined as follows: Given $v,w\in\VV$, then $(v,w)\in \PP_{\mathcal{D}}$ if and only if there exists a cluster $D \in \DD$ with $v, w \in D$. In words, $G$ is made up of $k$ disjoint cliques. It is only in this case that the similarity relationship defined through $G$ is transitive. Matrix $Y$ represents a clustering if, after permutation of rows and columns, it ends up being block-diagonal, where the $i$-th block is a $d_i\times d_i$ matrix of ones, $d_i$ being the size of the $i$-th cluster.
Given clustering $\DD$, we find it convenient to define a map $\mu_{\DD}\,:\,\VV \rightarrow \{1,\ldots,k\}$ in such a way that for all $v \in \VV$ we have $v \in D_{\mu_{\DD}(v)}$. In words, $\mu_{\DD}$ is a class assignment mapping, so that $v$ and $w$ are similar w.r.t. $\DD$ if and only if $\mu_{\DD}(v) = \mu_{\DD}(w)$. 

Given two similarity graphs $G = (\VV,\PP)$ and $G' = (\VV,\PP')$, the (Hamming error) distance between $G$ and $G'$, denoted here as $\SCsmt{\PP}{\PP'}$, is defined as
\[
\SCsmt{\PP}{\PP'} = \left|\{(v,w) \in \VV^2\,:\, (v,w) \in \PP \wedge (v,w) \notin \PP'
\vee  (v,w) \notin \PP \wedge (v,w) \in \PP' \} \right|\,,
\]
where $|A|$ is the cardinality of set $A$.
The same definition applies in particular to the case when either $G$ or $G'$ (or both) represent clusterings over $\VV$. By abuse of notation, if $\DD$ is a clustering and $G = (\VV,\PP)$ is a similarity graph, we will often write $\SCsmt{\DD}{\PP}$ to denote $\SCsmt{\PP_\DD}{\PP}$, where $(\VV,\PP_{\DD})$ is the similarity graph associated with $\DD$, so that $\SCsmt{\PP_\DD}{\DD} = 0$. Moreover, if the similarity graphs $G$ and $G'$ are represented by similarity matrices, we may equivalently write $\SCsmt{Y}{Y'}$, $\SCsmt{Y}{\DD}$, and so on. The quantity $\SCsmt{}{}$ is very closely related to the so-called Mirkin metric \cite{mir96} over clusterings, as well as to the (complement of the) Rand index \cite{ra71}, see, e.g., \cite{me11}. 

Another ``distance'' that applies specifically to clusterings is the misclassification error distance, denoted here as $\SCer{}{}$, and is defined as follows. Given two clusterings $\CC = \{C_1,\ldots,C_{\ell}\}$ and $\DD = \{D_1,\ldots,D_{k}\} $ over $\VV$, repeatedly add the empty set to the smaller of the two so as to obtain ${\ell} = k$. Then
\[
\SCer{\CC}{\DD} = \min_{f} \sum_{D \in \DD} |D\setminus f(D)|\,,
\] 
the minimum being over all bijections from $\DD$ to $\CC$. In words, $\SCer{\CC}{\DD}$ measures the smallest number of classification mistakes over all class assignments of clusters in $\DD$ w.r.t. clusters in $\CC$. This is basically an unnormalized version of the classification error distance considered, e.g., in \cite{me07}.


The (Jaccard) distance $\jdist(A,B)$ between sets $A$ and $B$, with $A,B \subseteq \VV$ is defined as
\[
\jdist(A,B) = \frac{|A\setminus B| + |B\setminus A|}{|A\cup B|}\,.
\]
Recall that $\jdist(,)$ is a proper metric on the collection of all finite sets. Moreover, observe that $\jdist(A,B) = 1$ if and only if $A$ and $B$ are disjoint.

Since our clustering algorithms will rely upon side information in the form of undirected graphs, we also need to recall relevant notions for such graphs and (spectral) properties thereof. Let  $Y$ be a similarity matrix and $G = (\VV,E)$ be a graph, henceforth called {\em side-information} graph. $G$ is assumed to be undirected, unweighted and connected.


As is standard in graph-based learning problems (e.g., \cite{Her08,HL09,HP07,HLP09,HPR09,CGV09b,CGVZ13,CGVZ10b,vcgz11,ghp13},
and references therein), graph $G$ encodes side information in that it suggests to the clustering algorithms that adjacent vertices in $G$ tend to be similar. The set of {\em cut-edges} in $G$ w.r.t. $Y$ is the set of edges $(v,w) \in E$ such that $y_{v,w} =0$, the associated cut-size (i.e., their number) will be denoted as $\Phi_{G}(Y)$ (or simply $\Phi_{G}$, if $Y$ is clear from the surrounding context). 

If $G$ is viewed as a resistive network where each edge is a unit resistor, then the {\em effective resistance} $r_{G}(v,w)$ of the pairing $(v,w) \in \VV^2$ is a measure of connectivity between the two nodes $v$ and $w$ in $G$ which, in the special case when $(v,w)\in E$, also equals the probability that a spanning tree of $G$ drawn uniformly at random from the set of all spanning trees of $G$ includes $(v,w)$ as one of its $n-1$ edges (e.g., \cite{lp10}). As a consequence, $\sum_{(v,w)\in E}r_{G}(v,w) = n-1$. Finally, $\Phi_{G}^R(Y)$ (or $\Phi_{G}^R$, for brevity) will denote the sum, over all cut-edges $(v,w)$ in $G$ w.r.t $Y$, of the effective resistances $r_{G}(v,w)$. This sum will sometimes be called the {\em resistance-weighted} cut-size of $G$ (w.r.t. $Y$). Notice that if $G$ is a tree we have $\Phi_{G}^R(Y) = \Phi_{G}(Y)$ for all $Y$.

The basic inductive principle underpinning \rgc\  is the assumption that $\Phi^R_{G}(Y)$ is small.\footnote
{
Notice that the edges in $G$ should not be considered as hard constraints (like the must-link constraints in semi-supervised clustering/clustering with side information, e.g., \cite{bdsy99,dbe99}). 
}
Both $\Phi^R_{G}$ and $\Phi_{G}$ can be considered as complexity measures for our learning problems, since they both depend on cut-edges in $E$.
However, unlike $\Phi_{G}$, the quantity $\Phi^R_{G}$ enjoys properties of {\em global} density-independence ($\Phi^R_{G}$ is at most $n-1$, hence it scales with the number of nodes of $G$ rather than the number of edges), and {\em local} density-independence ($\Phi^R_{G}$ suitably discriminates between dense and sparse graph topology areas -- see, e.g., the discussion in \cite{CGVZ13}). As such, $\Phi^R_{G}$ is more satisfactory than $\Phi_{G}$ in measuring the quality of side information at our disposal. 

\subsection{Learning setting}\label{ss:setting}
We are interested in inferring (or just computing) clusterings over $\VV$ based on binary similarity/dissimilarity information contained in a similarity matrix $Y$, possibly along with side information in the form of a connected and undirected graph $G = (\VV,E)$. The similarity matrix $Y$ itself may or may not represent a clustering over $\VV$. 
The error of our inference procedures will be measured through $\SCer{}{}$. We shall find bounds on $\SCer{}{}$ either directly, by presenting specific algorithms, or indirectly via (tight) reductions from similarity prediction problems/methods measured through $\SCsmt{}{}$ to clustering problems/methods measured through $\SCer{}{}$. 
More specifically, given a set of items $\VV = \{1,\ldots,n\}$ and a similarity matrix $Y$ representing a clustering $\DD$, our goal is to build a clustering $\CC$ over $\VV$ 
with as small as possible $\SCer{\CC}{\DD}$. 
We would like to do so by observing only a subset of the binary entries of $Y$. Notice that the number of clusters $k$ in the comparison clustering need not be known to the clustering algorithm.

In the setting of \rgc, we are given a side information graph $G = (\VV,E)$, and a training set $S$ of $m$ binary-labeled pairs $\langle (v,w), y_{u,v}\rangle \in\VV^2\times \{0,1\}$, drawn uniformly at random\footnote
{
For simplicity of presentation, we will assume the samples in $S$ are drawn from $\VV$ with replacement.
} 
from $\VV^2$. Our goal is to build a clustering $\CC$ over $\VV$ so as to achieve small misclassification error $\SCer{\CC}{Y}$, when this error is computed {\em on the whole} matrix $Y$. 

\section{Algorithms and Analysis}\label{s:algs}
We start off with a clustering algorithm that takes as input a similarity graph over $\VV$, and produces in output a clustering over $\VV$. This will be a building block for later results, but it can also be of independent interest.
\begin{algorithm}[t]
\caption{The Robust Greedy Clustering Algorithm\label{SCalgo4}}
{\bf Input:} Similarity graph $(\VV,\PP)$; distance parameter $a \in [0,1]$.
\begin{enumerate}
\item For all $v\in\VV$, set $\SCgam{v}\leftarrow\{v\}\cup\{w\in\VV:(v,w)\in\SCes\}$;
\item Construction of graph $(\VV, \SCfi)$:\hskip1ex \texttt{//First stage} \\ 
For all $v, w \in\VV$ with $v\neq w$:\\
If $\jdist\left(\SCgam{v},\SCgam{w}\right) \leq 1-a$ then $(i,j)\in\SCfi$, otherwise $(i,j)\notin\SCfi$;
\item Set $\SCtaken{1}\leftarrow\VV$, and $t\leftarrow 1$; \hskip5ex \texttt{//Second stage}
\item While $\SCtaken{t}\neq\emptyset$:
\begin{itemize}
\item For every $v\in\SCtaken{t}$ set $\SCnei{t}{v}\leftarrow\{v\}\cup\{w\in\SCtaken{t}:(v,w)\in\SCfi\}$,
\item Set $\SCsel{t}\leftarrow\operatorname{argmax}_{v\in\SCtaken{t}}|\SCnei{t}{v}|$,
\item Set $\SCclu{t}\leftarrow\SCnei{t}{\SCsel{t}}$,
\item Set $\SCtaken{t+1}\leftarrow\SCtaken{t}\setminus\SCclu{t}$,
\item $t\leftarrow t+1$;
\end{itemize}
\end{enumerate}
{\bf Output:} $\SCclu{1}, \SCclu{2}, ... , \SCclu{\SCnum}$, where $\SCnum = t-1$.
\end{algorithm}
Our algorithm, called Robust Greedy Clustering Algorithm (\rgc, for brevity), is displayed in Algorithm~\ref{SCalgo4}. The algorithm has two stages. The first stage is a robustifying stage where the similarity graph $(\VV,\SCes)$ is converted into a (more robust) similarity graph $(\VV,\SCfi)$ as follows: Given two distinct vertices $v,w\in\VV$, we have $(v,w)\in\SCfi$ if and only if the Jaccard distance of their neighbourhoods (in $(\VV,\SCes)$) is not bigger than $1-a$, for some distance parameter $a \in [0,1]$. The second stage uses a greedy method to convert the graph $(\VV,\SCfi)$ into a clustering $\SCcs$. This stage proceeds in ``rounds". At each round $t$ we have a set $\SCtaken{t}$ of all vertices which have not yet been assigned to any clusters. We then choose $\SCsel{t}$ to be the vertex in $\SCtaken{t}$ which has the maximum number of neighbours (under the graph $(\VV,\SCfi)$) in $\SCtaken{t}$, and take this set of neighbours (including $\SCsel{t}$) to be the next cluster.

From a computational standpoint, the second stage of \rgc\, runs in $\scO(n^2\log n)$ time, since on every round $t$ we single out $\alpha_t$ (which can be determined in $\log n$ time by maintaining a suitable heap data-structure), and erase all edges emanating from $\alpha_t$ in the similarity graph $(\VV,\QQ)$. On the other hand, the first stage of \rgc\, runs in $\scO(n^3)$ time, in the worst case, though standard techniques exist that avoid the all-pairs comparison, like a Locality Sensitive Hashing scheme applied to the Jaccard distance (e.g., \cite[Ch.3]{ru10}).
We have the following result.\footnote
{
All proofs are contained in the appendix.
}
\begin{theorem}\label{t:main2}
Let $\mathcal{C} = \{C_1,\ldots,C_k\}$ be the clustering produced in output by \rgc\, when receiving as input similarity graph $(\VV,\PP)$, and distance parameter $a = 2/3$. Then for any clustering $\DD = \{D_1,\ldots,D_k\}$, with $d_i = |D_i|$, $i = 1,\ldots, k$, and $d_1 \leq d_2 \leq \ldots d_k$ we have
\[
\SCer{\CC}{\DD}
\leq 
\min_{j=1,\ldots,k}\left(\frac{12}{d_j}\SCsmt{\SCes}{\SCtcs}
+\sum_{i=1}^{j-1}d_i\right)\,.
\]
\end{theorem}
Hence, if the chosen $\DD$ is the best approximation to $\PP$ w.r.t. $\SCsmt{}{}$, and we interpret $(\VV,\PP)$ as a noisy version of $\DD$, then small $\SCsmt{\SCes}{\SCtcs}$ implies small $\SCer{\CC}{\DD}$. In particular, $\SCsmt{\SCes}{\SCtcs} = 0$ implies $\SCer{\CC}{\DD} = 0$ (simply pick $j=1$ in the minimum). Yet, this result only applies to the case when the similarity graph $(\VV,\PP)$ is fully observed by our clustering algorithm. As we will see below, $(\VV,\PP)$ may in turn be the result of a similarity learning process when the similarity labels are provided by clustering $\DD$. In this sense, Theorem \ref{t:main2} will help us to deliver generalization bounds (as measured by $\SCer{\CC}{\DD}$), as a function of the generalization ability of this similarity learning process (as measured by $\SCsmt{\SCes}{\SCtcs}$).

The problem faced by \rgc\, is also related to the standard correlation clustering problem \cite{bbc04}. Yet, the goal here is somewhat different, since a correlation clustering algorithm takes as input $(\VV,\PP)$, but is aimed at producing a clustering $\CC$ such that $\SCsmt{\SCes}{\CC}$ is as small as possible.

In passing, we next show that the construction provided by \rgc\, is essentially optimal (up to multiplicative constants). Let $G_{\DD} = (\VV,E_{\DD})$ be the similarity graph associated with clustering $\DD$. We say that a clustering algorithm that takes as input a similarity graph over $\VV$ and gives in output a clustering over $\VV$ is {\em consistent} if and only if for every clustering $\mathcal{D}$ over $\VV$ the algorithm outputs $\mathcal{D}$ when receiving as input $G_{\DD}$. 
Observe that \rgc\, is an example of a consistent algorithm. We have the following lower bound.
%
\begin{theorem}\label{t:main3}
For any finite set $\VV$, any clustering $\SCtcs = \{D_1,D_2, \ldots, D_k\}$ over $\VV$, any positive constant $\sigma$, and any consistent clustering algorithm, there exists a similarity graph $(\VV,\SCes)$ such that $\SCsmt{\SCes}{\SCtcs}\leq\sigma$, while
\begin{equation}
\SCer{\mathcal{C}}{\mathcal{D}}
\geq 
\min_{j=1,\ldots,k}
\left(
\frac{1}{2d_{j}}\,\sigma-1+\frac{1}{4}\sum_{i=1}^{j-1}d_{i}
\right)\,,
\end{equation}
%
or $\SCer{\mathcal{C}}{\mathcal{D}}\geq\frac{n}{2}$, where $\SCcs$ is the output produced by the algorithm when given $(\VV, \SCes)$ as input, and
$d_i = |D_i|$, $i = 1,\ldots, k$, with $d_1 \leq d_2 \leq \ldots d_k$.
\end{theorem}
From the proof provided in the appendix, one can see that the similarity graph $(\VV,\PP)$ used here is indeed a {\em clustering} over $\VV$ so that, as the algorithm is consistent, the output $\CC$ must be such a clustering. This result can therefore be contrasted to the results contained, e.g., in \cite{me11} about the equivalence between clustering distances, specifically Theorem 26 therein. Translated into our notation, that result reads as follows: $\SCer{\mathcal{C}}{\mathcal{D}} \geq \frac{\SCsmt{\SCes}{\SCtcs}}{16d_k}$. Our Theorem \ref{t:main3} is thus sharper but, unlike the one in \cite{me11}, it {\em does not apply} to any possible pairs of clusterings $\CC$ and $\DD$, for in our case $\CC$ is selected as a function of $\DD$.

\subsection{Learning to Cluster}\label{ss:passive}
Suppose now that our clustering algorithm has at its disposal a side information graph $G$, and a training set $S$ of size $m$. Training set $S$ is drawn at random from $\VV^2$, and is labeled according to a similarity matrix $Y$ representing a clustering $\DD = \{D_1,\ldots,D_k\}$ with cluster sizes $d_i = |D_i|$, $i = 1,\ldots,k$, and having resistance-weighted cutsize $\Phi_{G}^R(Y)$. A Laplacian-regularized Matrix Winnow algorithm \cite{wa07}, as presented in \cite{ghp13}, is an online algorithm that sweeps over $S$ only once, and is guaranteed to make $\scO(\Phi_{G}^R\log^3 n)$ many mistakes in expectation (see Theorem 5 therein). In turn, this algorithm can be used within an online-to-batch conversion wrapper, like the one mentioned in \cite{hw95}, or the one in \cite{cg08} to produce a similarity graph $(\VV,\PP)$ (which need not be a clustering) such that
\[
\E\SCsmt{\SCes}{Y} = \scO\left(\frac{n^2}{m}\,\Phi_{G}^R\log^3 n \right)\,.
\]
Then, in order to produce a "good" clustering $\CC$ out of $\PP$, we can apply \rgc\ to input $(\VV,\PP)$. Invoking Theorem \ref{t:main2}, we conclude that
%
\begin{align}\label{e:compound}
\E\SCer{\CC}{\DD} 
&\leq \E\left[ \min_{j=1,\ldots,k}
\left(\frac{12}{d_j} \SCsmt{\SCes}{\SCtcs}
+\sum_{i=1}^{j-1}d_i\right)\right]\notag\\
&\leq \min_{j=1,\ldots,k}
\left(\frac{12}{d_j}\E\SCsmt{\SCes}{\SCtcs}
+\sum_{i=1}^{j-1}d_i\right)\notag\\
&= \scO\left( \min_{j=1,\ldots,k}
\left(\frac{1}{d_j}\frac{n^2}{m}\,\Phi_{G}^R\log^3 n
+\sum_{i=1}^{j-1}d_i\right)  \right)\,.
\end{align}
The training time of the whole procedure is dominated by the $\scO(n^3)$ time per round required by Matrix Winnow, which is thus $\scO(mn^3)$. In what follows, we take a more direct (and time-efficient) route to obtain alternative statistical guarantees in the simplest case when the side-information graph is absent.

\begin{algorithm}[t]
{\bf Input:} Item set $\VV = \{1,\ldots,n\}$; training set $S$.

\begin{enumerate}
\item Initialization:
$\CC = \{\{1\}, \ldots, \{n\}\}$\,;
\item For any $v \in \VV$, let $C_v$ denote the cluster of $\scC$ containing $v$\,;
\item For each $(v,w)\in \PTr$:
 
      \qquad If $(y_{v,w}=1)$ $\land$ $(C_v\not\equiv C_w)$ 
      then
      $\scC \leftarrow \scC \setminus C_w$\quad and \quad $C_v \leftarrow C_v \cup C_w$\,;
\end{enumerate}
{\bf Output: } Clustering $\scC$\,.
    \caption{The Simple Agglomerative Clustering Algorithm.\label{alg:lbanca}}
\end{algorithm}






Algorithm \ref{alg:lbanca} displays the pseudocode of \bnc\ (Simple Agglomerative Clustering Algorithm).
\bnc\ takes as input the item set $V$ and a training set $S$.
The algorithm operates as follows. 
It starts by assigning a different cluster to each vertex in $\VV$, and sequentially inspects each $\langle(v,w),y_{v,w} \rangle \in \PTr$ aiming to merge clusters. In particular, whenever $v$ and $w$ currently fall into different clusters but $y_{v,w}=1$, the two clusters are merged, as in a standard agglomerative clustering procedure. Finally, \bnc\ outputs the clustering $\scC$ so computed. Notice that, unlike the Matrix Winnow-based algorithm, no side-information in the form of a graph over $V$ is exploited.
%
%


%

The following theorem quantifies the performance of \bnc. 
%
\begin{theorem}\label{th:batch}
Given similarity matrix $Y$ encoding a clustering over $\VV$ with $k$ clusters, 
\bnc\ 
returns a clustering $\scC$ such that $\SCer{\scC}{Y}$ is bounded as
\[
\E\left[\SCer{\scC}{Y}\right] = \scO\left(\frac{n^2}{m}\,k\log \frac{n^2}{m}\right)\,,
\]
the expectation being over a random draw of $S$.
\end{theorem}

It is instructive to compare the upper bounds contained in Theorem \ref{th:batch} to the one in Eq. (\ref{e:compound}). The two bounds 
are in general incomparable. While $\Phi_{G}^R$ is always at least as large as $k-1$ (recall that $G$ is connected), the bound in Theorem \ref{th:batch} does also depend in a detailed way on the sizes $d_i$ of the underlying clustering $\DD$. For instance, if $d_i = n/k$ for all $i$ then (\ref{e:compound}) is sharper in the presence of informative side-information than the bound in Theorem \ref{th:batch}. This is because, up to log factors, the resulting bound is of order
\(
\frac{n\,k}{m}\Phi_{G}^R
\)
which is no larger than the bound in Theorem \ref{th:batch} since $k-1 \le \Phi_{G}^R \leq n-1$.  Thus in the case of maximally informative side-information ($\Phi_{G}^R = \Theta(k)$) and balanced cluster sizes the bound is improved by a factor of $\frac{k}{n}$.
%
On the other hand, \bnc\ is definitely much faster than the Matrix Winnow-based algorithm since, apart from the random spanning tree construction, it only takes $\scO((n+m)\log^* n)$ time to run if implemented via standard (union-find) data-structures, where
$\log^* n$ is the iterated logarithm of $n$.

We complement the two upper bounds with the following lower bound result, showing that the dependence of $\SCer{}{}$ on $\Phi_G^R$ (or $k$) cannot be eliminated.
\begin{theorem}\label{th:batchlb}
Given any side-information graph $G=(\VV,E)$, any $b \in [4, $n$-1]$, any $k > 2$ and any $m < \frac{n^2}{4}$, there exists a 
similarity matrix $Y$ representing a clustering formed by at most $k$ clusters such that for any algorithm giving in output clustering $\CC$ we have
$\SCer{\CC}{Y} = \Omega\left(\min\left\{\frac{n^2}{m}\,k, b\right\}\right)$ 
while $\Phi^R_G(Y) \leq b$.
\end{theorem}

\section{Conclusions and Ongoing Research}\label{s:conc}
We have investigated the problem of learning a clustering over a finite set from pairwise training data and side-information data. Two routes have been followed to tackle this problem: i. a direct route, where we exhibited a specific algorithm, called \bnc, operating without side information graph, and ii. an indirect route that steps through a reduction, called \rgc, establishing a tight bridge between two clustering metrics, that takes the side-information graph into account. 
We provided two misclassification error analyses in the case when the source of similarity data is consistent with a given clustering, and complemented these analyses with an involved construction delivering an almost matching lower bound.

Two extensions we are currently exploring are: 
i. extending the underlying statistical assumptions on data (e.g., sampling distribution-free guarantees) while retaining running time efficiency, and
ii. studying other learning regimes, like active learning, under similar or broader statistical assumptions as those currently in this paper.

\vspace{.2truecm}

{\bf Acknowledgements.}\ {\small
This work was supported in part by the U.S. Army Research Laboratory and the U.K. Defence Science and Technology Laboratory and was accomplished under Agreement Number W911NF-16-3-0001. The views and conclusions contained in this document are those of the authors and should not be interpreted as representing the official policies, ether expressed or implied, of the U.S. Army Research Laboratory, the U.S. Government, the U.K. Defence Science and Technology Laboratory or the U.K. Government. The U.S. and U.K. Governments are authorized to reproduce and distribute reprints for Government purposes notwithstanding any copyright notation herein.}

\bibliographystyle{plain}

\appendix

\newpage

\section{Proofs}\label{as:proofs}

\subsection{Proof of Theorem \ref{t:main2}}
The following two lemmas are an immediate consequence of the triangle inequality for $\jdist$.
\begin{lemma}\label{SClem1}
Let $a,b \in [0,1]$ be such that $a+b \geq 3/2$, and sets $U, W, X, Y, Z$ satisfy
\begin{enumerate} 
\item \label{SCIL3} $\jdist(U,W)\leq 1-a$;
\item \label{SCIL4} $\jdist(W,X)\leq 1-b$;
\item \label{SCIL5} $\jdist(U,Y)\leq 1-a$;
\item \label{SCIL6} $\jdist(Z,X) =1$.
\end{enumerate}
Then $\jdist(Y,Z)\geq 1-b$.
\end{lemma}
\begin{proof}
We can write
\begin{align*}
1 
&= \jdist(Z,X)\\ 
&\leq \jdist(Z,Y) + \jdist(Y,U) + \jdist(U,W) + \jdist(W,X)\\
&\leq  \jdist(Z,Y) + 1-a + 1-a + 1-b~,
\end{align*}
so that 
\[
\jdist(Z,Y) \geq 2a+b-2 \geq 1-b,
\]
the last inequality using the assumption $a+b \geq 3/2$. This concludes the proof.
\end{proof}
\begin{lemma}\label{SClem3}
Let $a,b \in [0,1]$ be such that $2b \geq 1+a$, and sets $X, Y, Z$ satisfy
\begin{enumerate}
\item \label{SCLI7} $\jdist(X,Y) \leq 1-b$:
\item \label{SCLI8} $\jdist(Y,Z) \geq 1-a$~.
\end{enumerate}
Then $\jdist(X,Z) \geq 1-b$.
\end{lemma}
\begin{proof}
We can write
\[
1-a \leq \jdist(Y,Z) \leq \jdist(Y,X) + \jdist(X,Z) \leq 1-b + \jdist(X,Z),
\]
so that
\[
\jdist(X,Z) \geq b-a \geq 1-b,
\] 
the last inequality deriving from $2b \geq 1+a$. This concludes the proof.
\end{proof}
With these two simple lemmas handy, we are now ready to analyze \rgc. The reader is compelled to refer to Algorithm \ref {SCalgo4} for notation. In what follows, $a \in [0,1]$ is \rgc's distance parameter, and $b \in [0,1]$ is a constant such that the two conditions on $a$ and $b$ required by Lemmas \ref{SClem1} and \ref{SClem3} simultaneously hold. It is easy to see that these conditions are equivalent to\footnote
{
For instance, we may set $a = 2/3$ and $b = 5/6$.
} 
\begin{equation}\label{e:abcond}
a \geq \frac{2}{3},\qquad b \geq \frac{1+a}{2}\,.
\end{equation}

The following definition will be useful.
\begin{definition}
A $b$-\textit{anomaly} in the similarity graph $(\VV,\PP)$ is a vertex $v\in\VV$ for which 
$\jdist(\SCtclu{\SClab{v}},\SCgam{v}) \geq 1-b$, for some constant $b \in [0,1]$ satisfying (\ref{e:abcond}). We denote by $\SCano$ the set of all anomalies. A \textit{centered round} of \rgc\ is any $t\leq \SCnum$ in which $\SCnei{t}{\alpha_t}\not\subseteq\SCano$. We denote by $\SCcen$ the set of all centered rounds. A \textit{centered label} is any class $i\in \{1,\ldots,k\}$ such that $\SCtclu{i}\not\subseteq\SCano$. We denote by $\SCus$ the set of all centered labels.
\end{definition}
\begin{lemma}\label{SClem2}
For any round $t \leq \SCnum$, there exists a class $j\in \{1,\ldots, k\}$ such that for every vertex $v\in\SCnei{t}{\alpha_t}\setminus\SCano$ we have $\SClab{v}=j$.
\end{lemma}
\begin{proof}
Suppose, for the sake of contradiction, that we have $v,w\in \SCnei{t}{\alpha_t}$ with $v, w\notin\SCano$ and $\SClab{v}\neq\SClab{w}$. Define $U:=\SCgam{\alpha_t}$, $W:=\SCgam{v}$, $X:=\SCtclu{\SClab{v}}$, $Y:=\SCgam{w}$, and $Z:=\SCtclu{\SClab{v}}$.
Since $v,w\in\SCnei{t}{\alpha_t}$, by the way graph $(\VV,\QQ)$ is constructed, we have both $\jdist(U,W)\leq 1-a$ and $\jdist(U,Y)\leq 1-a$. Moreover, since $v\notin\SCano$ we have $\jdist(W,X) < 1-b$. Also, $\SClab{v}\neq\SClab{w}$ implies $\jdist(Z,X) = 1$. We are therefore in a position to apply Lemma \ref{SClem1} verbatim, from which we have $\jdist(Y,Z)\geq 1-b$, i.e., $w\in\SCano$. This is a contradiction, which implies the claimed result.
\end{proof}
Lemma \ref{SClem2} allows us to make the following definition.

\begin{definition}
Given a centered round $t\in\SCcen$, we define $\gamma(t)$ to be the unique class $j$ such that for every vertex $v\in\SCnei{t}{\alpha_t}\setminus\SCano$ we have $\SClab{v}=j$.
\end{definition} 
\begin{lemma}\label{SClem4}
For any round $t \leq \ell$ and vertices $v,w\in A_t$ with $v\notin\SCano$, $w\notin\SCnei{t}{v}$ and $\SClab{v}=\SClab{w}$ we have $w\in\SCano$.
\end{lemma}
\begin{proof}
Define $X:=\SCtclu{\SClab{v}}$, $Y:=\SCgam{v}$ and $Z:=\SCgam{w}$. Since $v\notin{\SCano}$ we have $\jdist(X,Y) \leq 1-b$. Moreover, $w\notin\SCnei{t}{v}$ implies $\jdist(Y,Z) \geq 1-a$. By Lemma \ref{SClem3} we immediately have $\jdist(X,Z) \geq 1-b$. But since $X=\SCtclu{\SClab{v}}=\SCtclu{\SClab{w}}$, this equivalently establishes that $w\in\SCano$.
\end{proof}
\begin{lemma}\label{SClem5}
For any centered round $t\in\SCcen$, any vertex $v\in\SCnei{t}{\alpha_t}\setminus\SCano$, and any vertex $w\in\SCnei{t}{\alpha_t}\setminus\SCnei{t}{v}$, we have $w\in\SCano$.
\end{lemma}
\begin{proof}
If $\SClab{w}\neq\SClab{v}$ then by Lemma \ref{SClem2} we must have $w\in\SCano$, so we are done. On the other hand, if $\SClab{w}=\SClab{v}$, we have $v\notin\SCano$, $w\notin\SCnei{t}{v}$ and $\SClab{v}=\SClab{w}$ which implies, by Lemma \ref{SClem4}, that $w\in\SCano$.
\end{proof}
\begin{lemma}
\label{SClem8}
For any centered round $t\in\SCcen$, we have 
$|(\SCtaken{t+1}\cap\SCtclu{\gamma(t)})\setminus\SCano| \leq |\SCclu{t}\cap\SCano|$.
\end{lemma}
\begin{proof}
Since $t\in\SCcen$ there must exist a vertex $v\in\SCnei{t}{\alpha_t}$ with $v\notin\SCano$, so let us consider such a $v$. Note that by the way the algorithm works,
we have $|\SCnei{t}{\alpha_t}|\geq|\SCnei{t}{v}|$, so that
$|\SCnei{t}{v}\setminus\SCnei{t}{\alpha_t}|\leq|\SCnei{t}{\alpha_t}\setminus\SCnei{t}{v}|$. 
Next, by Lemma \ref{SClem5} we have $\SCnei{t}{\alpha_t}\setminus\SCnei{t}{v}\subseteq\SCano$,
hence 
$\SCnei{t}{\alpha_t}\setminus\SCnei{t}{v} \subseteq \SCnei{t}{\alpha_t}\cap\SCano$ and, consequently, $|\SCnei{t}{\alpha_t}\setminus\SCnei{t}{v}|\leq |\SCnei{t}{\alpha_t}\cap\SCano|$. Recalling that $\SCclu{t}=\SCnei{t}{\alpha_t}$, we have therefore obtained
\begin{equation}\label{e:second}
|\SCnei{t}{v}\setminus\SCnei{t}{\alpha_t}|\leq|\SCnei{t}{\alpha_t}\setminus\SCnei{t}{v}|
\leq 
|\SCnei{t}{\alpha_t}\cap\SCano|=|\SCclu{t}\cap\SCano|\,.
\end{equation}
Now suppose we have some vertex $w\in(\SCtaken{t+1}\cap\SCtclu{\gamma(t)})\setminus\SCano$.
For the sake of contradiction, let us assume that $w\notin\SCnei{t}{v}$. Then 
$w\in\SCtaken{t+1}$ implies $w\in\SCtaken{t}$ which, combined with Lemma \ref{SClem4} together with the fact that $\SClab{v}=\SCc{t}=\SClab{w}$, implies that $w\in\SCano$, which is a contradiction. Hence we must have $w\in\SCnei{t}{v}$. Moreover, since $w\in\SCtaken{t+1}$ we must have $w\notin\SCnei{t}{\alpha_t}$. 
We have hence shown that $w\in\SCnei{t}{v}\setminus\SCnei{t}{\alpha_t}$, implying that 
$|(\SCtaken{t+1}\cap\SCtclu{\gamma(t)})\setminus\SCano| \leq |\SCnei{t}{v}\setminus\SCnei{t}{\alpha_t}|$. 
Combining with (\ref{e:second}) concludes the proof.
\end{proof}
We now turn to considering centered labels.
\begin{lemma}\label{SClem6}
For any centered label $i\in\SCus$ there exists some round $t\leq\SCnum$ such that $\SCc{t}=i$.
\end{lemma}
\begin{proof}
Since $i$ is a centred label, pick $v\in\SCtclu{i}\setminus\SCano$, Further, since $\SCclu{1}, \SCclu{2}, ..., \SCclu{\SCnum}$ is a partition of $\SCver$, choose $t$ such that $v\in\SCclu{t}$. Now, since $v\in\SCclu{t}\setminus{\SCano}$ we have that $t\in\SCcen$ and, by Lemma \ref{SClem2}, that $\SCc{t}=\SClab{v}=i$.
\end{proof}
Lemma \ref{SClem6} allows us to make the following definition.
\begin{definition}
Given a centered label $i\in\SCus$, we define $\SCmat{i}:=\min\{t\,:\,\SCc{t}=i\}$.
\end{definition}
\begin{lemma}\label{SClem7}
For any centered label $i\in\SCus$, we have $\SCtclu{i}\setminus\SCano\subseteq\SCtaken{\SCmat{i}}$.
\end{lemma}
\begin{proof}
Suppose, for contradiction, that there exists some $v\in\SCtclu{i}\setminus\SCano$ with $v\notin\SCtaken{\SCmat{i}}$. Then, by definition of $\SCtaken{\SCmat{i}}$ there exists some round $t^o<\SCmat{i}$ with $v\in\SCclu{t^o}$. As $v\notin\SCano$ we have $t^o\in\SCcen$ and, by Lemma \ref{SClem2}, that $\SClab{v}=\SCc{t^o}$. Hence $\SCc{t^o}=\SClab{v}=i$ which, due to the condition $t^o<\SCmat{i}$, contradicts the fact that $\SCmat{i}:=\min\{t:\SCc{t}=i\}$.
\end{proof}
\begin{lemma}
\label{SClem9}
For any centred label $i\in\SCus$ we have $|\SCtclu{i}\setminus\SCclu{\SCmat{i}}|\leq|\SCtclu{i}\cap\SCano|+|\SCclu{\SCmat{i}}\cap\SCano|$.
\end{lemma}
\begin{proof}
Suppose we have some $v\in\SCtclu{i}\setminus\SCclu{\SCmat{i}}$, and let us separate the two cases: (i) $v\notin\SCano$ and, (ii) $v\in\SCano$.

Case (i). Since $v\in\SCtclu{i}\setminus\SCano$ we have, by Lemma \ref{SClem7}, that $v\in\SCtaken{\SCmat{i}}$. Since $v\notin\SCclu{\SCmat{i}}$ this implies that $v\in\SCtaken{\SCmat{i}+1}$. Notice that $\SCc{\SCmat{i}}=i$ so $\SCtclu{i}=\SCtclu{\SCc{\SCmat{i}}}$ and hence $v\in(\SCtaken{\SCmat{i}+1}\cap\SCtclu{\SCc{\SCmat{i}}})\setminus\SCano$. By Lemma \ref{SClem8} the number of such vertices $v$ is hence upper bounded by 
$|\SCclu{\SCmat{i}}\cap\SCano|$.

Case (ii). In this case, we simply have that
$v\in\SCtclu{i}\cap\SCano$, so the number of such vertices $v$ is upper bounded by 
$|\SCtclu{i}\cap\SCano|$.

Putting the two cases together gives us 
$|\SCtclu{i}\setminus\SCclu{\SCmat{i}}|
\leq
|\SCtclu{i}\cap\SCano|+ |\SCclu{\SCmat{i}}\cap\SCano|$, 
as required.
\end{proof}
Having established the main building blocks of the behavior of \rgc, we now turn to quantifying the resulting connection between ER and HA. 
To this effect, we start off by defining a natural map $\SCinj$ associated with the clustering $\{C_1,\ldots,C_{\ell}\}$ generated by \rgc, along with a corresponding accuracy measure.
\begin{definition}\label{d:clustermap}
The map $\SCinj\,:\, \{D_1, \ldots, D_k\} \rightarrow \{C_1,\ldots,C_{\ell}\}$ is defined as follows: 
\[
\SCinj(\SCtclu{i}) = 
\begin{cases}
\SCclu{\SCmat{i}} &{\mbox{if $i\in\SCus$}}\\
\emptyset         &{\mbox{if $i\notin\SCus$}}
\end{cases}
\]
Moreover, let $\SCmis{\SCinj}:=\sum_{i=1}^k |\SCtclu{i}\setminus\SCinj(\SCtclu{i})|$.
\end{definition}
We have the following lemma.
\begin{lemma}\label{SClem10}
$\SCmis{\SCinj}\leq 2|\SCano|$.
\end{lemma}
\begin{proof}
For $i\notin\SCus$ we have $\SCtclu{i}\subseteq\SCano$ and $\SCinj(\SCtclu{i})=\emptyset$ so that 
\[
|\SCtclu{i}\setminus\SCinj(\SCtclu{i})|=|\SCtclu{i}|=|\SCtclu{i}\cap\SCano|=|\SCtclu{i}\cap\SCano|+|\emptyset|=|\SCtclu{i}\cap\SCano|+|\SCinj(\SCtclu{i})\cap\SCano|.
\] 
On the other hand, for $i\in\SCus$ we have $\SCinj(\SCtclu{i})=\SCclu{\SCmat{i}}$ so that, by Lemma \ref{SClem9}, we can write 
\[
|\SCtclu{i}\setminus\SCinj(\SCtclu{i})|\leq|\SCtclu{i}\cap\SCano| + |\SCinj(\SCtclu{i})\cap\SCano|\,.
\] 
Hence, in both cases, for all $i \in \{1,\ldots,k \}$ we have 
\[
|\SCtclu{i}\setminus\SCinj(\SCtclu{i})|
\leq|
\SCtclu{i}\cap\SCano|+|\SCinj(\SCtclu{i})\cap\SCano|\,,
\]
implying
\begin{equation}\label{e:third}
\SCmis{\SCinj} \leq 
\sum_{i=1}^k \Bigl(|\SCtclu{i}\cap\SCano|+|\SCinj(\SCtclu{i})\cap\SCano|\Bigl)\,.
\end{equation}
Now, both $\{\SCtclu{1},\ldots, \SCtclu{k}\}$ and $\{\SCinj(\SCtclu{1}),\ldots, \SCinj(\SCtclu{k})\}$ are a partition of $\VV$, implying
\[
|\SCano| 
= 
\sum_{i=1}^k|\SCtclu{i}\cap\SCano| 
= 
\sum_{i=1}^k|\SCinj(\SCtclu{i})\cap\SCano|\,.
\]
Plugging back into (\ref{e:third}) yields the claimed result.
\end{proof}
Next, observe that, by its very definition, $\SCsmt{\PP}{\DD}$ can be rewritten as
\begin{equation}\label{e:fourth}
\SCsmt{\PP}{\DD} 
= 
\sum_{v\in\VV}|(\SCtclu{\SClab{v}}\setminus\SCgam{v})\cup(\SCgam{v}\setminus\SCtclu{\SClab{v}})|\,.
\end{equation}
\begin{lemma}\label{SClem12}
We have
$\SCsmt{\SCes}{\SCtcs}\geq (1-b)\sum_{i=1}^{\SCtnum}d_i|\SCtclu{i}\cap\SCano|$\,.
\end{lemma}
\begin{proof}
Fix class $i \in \{1,\ldots,k\}$ and vertex $v\in\SCtclu{i}\cap\SCano$. Then $v\in\SCano$ implies 
$\jdist(\SCtclu{\SClab{v}},\SCgam{v})\geq 1-b$, which in turn yields 
\[
|(\SCtclu{\SClab{v}}\setminus\SCgam{v})\cup(\SCgam{v}\setminus\SCtclu{\SClab{v}})| 
\geq (1-b)d_i\,,
\]
thereby concluding that for all fixed $i$ 
\[
\sum_{v\in\SCtclu{i}\cap\SCano}
|(\SCtclu{\SClab{v}}\setminus\SCgam{v})\cup(\SCgam{v}\setminus\SCtclu{\SClab{v}})|    
\geq (1-b)\,|\SCtclu{i}\cap\SCano|\,d_i\,. 
\]
Since $\SCano=\bigcup_{i=1}^k (\SCtclu{i}\cap\SCano)$, being the sets $\SCtclu{i}\cap\SCano$, $i = 1, \ldots, k$, pairwise disjoint, we can write
\begin{align*}
\sum_{v\in\SCano}
|(\SCtclu{\SClab{v}}\setminus\SCgam{v})\cup(\SCgam{v}\setminus\SCtclu{\SClab{v}})|
&=
\sum_{i=1}^{k}\sum_{v\in\SCtclu{i}\cap\SCano}
|(\SCtclu{\SClab{v}}\setminus\SCgam{v})\cup(\SCgam{v}\setminus\SCtclu{\SClab{v}})|\\
&\geq
(1-b)\,\sum_{i=1}^{k} |\SCtclu{i}\cap\SCano|\,d_i\,.
\end{align*}
Thus, from (\ref{e:fourth}), and the fact that $\SCano\subseteq\SCver$ the result immediately follows.
\end{proof}
\begin{lemma}
\label{SClem13}
The number $|\SCano|$ of $b$-anomalies can be upper bounded as
\[
|\SCano| \leq \min_{j = 1,\ldots,k} 
\left( \frac{1}{d_j(1-b)}\SCsmt{\SCes}{\SCtcs}+\sum_{i=1}^{j-1}d_i \right)\,.
\]
\end{lemma}
\begin{proof}
For any $j = 1, \ldots,k$ we can write
\[
|\SCano|=\sum_{i=1}^k|\SCtclu{i}\cap\SCano|
=
\sum_{i=1}^{j-1}|\SCtclu{i}\cap\SCano|+\sum_{i=j}^{\SCtnum}|\SCtclu{i}\cap\SCano|
\leq
\sum_{i=1}^{j-1}d_i+\sum_{i=j}^{\SCtnum}|\SCtclu{i}\cap\SCano|
\]
so all that is left to prove is that the last sum in the right-hand side is at most
$\frac{1}{d_j(1-b)}\SCsmt{\SCes}{\SCtcs}$.

Since, for all classes $i$ such that $i\geq j$, we have $d_i\geq d_j$, we can write
\[
\sum_{i=j}^{k}|\SCtclu{i}\cap\SCano|
\leq
\sum_{i=j}^{k}\frac{d_i}{d_j}|\SCtclu{i}\cap\SCano|
\leq
\frac{1}{d_j}\sum_{i=1}^{k} d_i\,|\SCtclu{i}\cap\SCano|
\leq 
\frac{1}{d_j(1-b)}\SCsmt{\SCes}{\SCtcs}\,,
\]
where the last inequality derives from Lemma \ref{SClem12}. This concludes the proof.
\end{proof}
We are now ready to combine to above lemmas into the proof of Theorem \ref{t:main2}.
\begin{proof}(Theorem \ref{t:main2})
Direct from Lemmas \ref{SClem10} and \ref{SClem13} we have 
\[
\SCmis{\SCinj}
\leq 
\min_{j=1,\ldots,k}\left(\frac{2}{d_j\,(1-b)}\SCsmt{\SCes}{\SCtcs}
+\sum_{i=1}^{j-1}d_i\right)\,.
\]
We then optimize for $b$ by selecting $b=\frac{1+a}{2}$, and then for $a$ by setting $a = 2/3$, so as to fulfil conditions (\ref{e:abcond}). The result follows by the fact that 
$\SCer{\mathcal{C}}{\mathcal{D}} \leq \SCmis{\SCinj}$, 
for $\SCer{\mathcal{C}}{\mathcal{D}}$ is a minimum over all possible cluster maps $\mathcal{D} \rightarrow \mathcal{C}$, while $\SCinj$ is just the one in Definition \ref{d:clustermap}.
\end{proof}

\subsection{Proof of Theorem \ref{t:main3}}

\begin{proof}
For ease of proof, we assume that $d_j$ is even for all $j$ (adapting the proof to the general case is trivial). We consider two cases:
\begin{enumerate}
\item $\sigma\geq\frac{1}{2}\sum_{j=1}^k{d_j^2}$;
\item $\sigma<\frac{1}{2}\sum_{j=1}^k{d_j^2}$\,.
\end{enumerate}
For the first case we choose, for every $j = 1,\ldots,k$, sets $P_j^+$ and $P_j^-$ such that $|P_j^+|=|P_j^-|=d_j/2$ and $P_j^+\cup P_j^-=D_j$. We then construct the similarity graph $(\SCver,E_{\SCes})$, where clustering $\SCes$ is made up of the $2k$ clusters
$\{P_j^+\,:\,j = 1,\ldots,k\}\cup \{P_j^-\,:\,j = 1,\ldots,k\}$. 
Since the algorithm is consistent, we must have $\SCcs=\SCes$. Now, let $f$ be an injection from $\SCtcs$ to $\SCcs$, and consider any $j = 1,\ldots,k$.
If $f(\SCtclu{j})\in\{P_j^+, P_j^-\}$ then we have $|\SCtclu{j}\setminus f(\SCtclu{j})|=d_j/2$, and otherwise $|\SCtclu{j}\setminus f(\SCtclu{j})|=d_{j}$, so that 
\[
\sum_{j=1}^k|\SCtclu{j}\setminus f(\SCtclu{j})| 
\geq
\frac{1}{2}\sum_{j=1}^k d_{j}= n/2\,.
\]
Since $f$ is arbitrary, this shows that $\SCer{\mathcal{C}}{\mathcal{D}}\geq\frac{n}{2}$.
Moreover, we observe that the only incorrect similarity/dissimilarity predictions of $\SCes$ with respect to $\SCtcs$ are those between $P_j^+$ and $P_j^-$, for every $j$, which gives us $2|P_j^+|\cdot|P_j^-|=d_{j}^2/2$ incorrect predictions for every $j$. This implies that 
$\SCsmt{\SCes}{\SCtcs}
=
\sum_{j=1}^k{d_j^2}/{2}$, 
which is no greater than $\sigma$, thereby completing the proof for the first case.

We now turn to the second case. Let $j^o \in \{1,\ldots,k\}$ be such that 
\[
\frac{1}{2}\sum_{i=1}^{j^o-1}{d_i^2} 
\leq 
\sigma 
< 
\frac{1}{2}\sum_{i=1}^{j^o}{d_{i}^2},
\]
and $\omega:= \sigma-\frac{1}{2}\sum_{i=1}^{j^o-1}{d_i^2}$. Notice that $\omega\leq d_{j^o}^2/2$. 
We choose, for every $j< j^o$, sets $P_j^+$ and $P_j^-$ such that 
$|P_j^+|=|P_j^-| = d_j/2$ and $P_j^+\cup P_j^- = D_j$. 
Let $c=\lfloor\omega/2d_{j^o}\rfloor$, and note that $c\leq d_{j^o}/4<d_{j^o}/2$. 
We can hence define subsets $X, Y\subseteq D_{j^o}$ such that 
$|X|=c$, $X\cup Y=D_{j^o}$ and $X\cap Y=\emptyset$. 

We construct the similarity graph $(\SCver,E_{\SCes})$, where clustering $\SCes$ is made up of the $k+j^o$ clusters
\[
\{P_j^+\,:\,j = 1,\ldots,j^o-1\}\cup \{P_j^-\,:\,j = 1,\ldots,j^o-1\}
\cup
\{X,Y\}
\cup
\{D_j\,:\,j > j^o\}\,.
\] 
Again, since the algorithm is consistent, we must have $\SCcs=\SCes$. As before, let $f$ be an arbitrary injection from $\SCtcs$ to $\SCcs$, and consider any $j < j^o$. Then if
$f(D_j)\in\{P_j^+, P_j^-\}$ we have $|D_j\setminus f(D_j)|=d_j/2$, otherwise 
$|D_j\setminus f(D_j)|=d_j$, so that $|D_j\setminus f(D_j)|\geq d_j/2$ holds for any $j < j^o$.  
Further, 
if $f(D_{j^o})=X$ then $|D_{j^o}\setminus f(D_{j^o})|=d_{j^o}-c$, 
if $f(D_{j^o})=Y$ then $|D_{j^o}\setminus f(D_{j^o})|=c$, 
and otherwise 
$|D_{j^o}\setminus f(D_{j^o})|=d_{j^o}$. 
In any case, since $c < d_{j^o}/2$, we have 
$|D_{j^o}\setminus f(D_{j^o})|\geq c$. This allows us to conclude that
\begin{align*}
\SCer{\mathcal{C}}{\mathcal{D}} 
&= 
\sum_{j=1}^k |D_{j}\setminus f(D_{j})|\\
&\geq
c + \frac{1}{2}\,\sum_{j=1}^{j^o-1} d_j\\
&= 
\lfloor \omega/2d_{j^o}\rfloor + \frac{1}{2}\,\sum_{j=1}^{j^o-1} d_j\\
&\geq 
\frac{\omega}{2d_{j^o}} - 1 + \frac{1}{2}\,\sum_{j=1}^{j^o-1} d_j\\
&= 
\frac{\sigma}{2d_{j^o}} - 1 - \frac{1}{4d_j^o}\,\sum_{j=1}^{j^o-1} d_j^2
+ \frac{1}{2}\,\sum_{j=1}^{j^o-1} d_j\\
&= 
\frac{\sigma}{2d_{j^o}} - 1 + \frac{1}{2}\,\sum_{j=1}^{j^o-1} d_j\left(1-\frac{d_j}{2d_{j^o}} \right)\\
&\geq
\frac{\sigma}{2d_{j^o}} - 1 + \frac{1}{4}\,\sum_{j=1}^{j^o-1} d_j\,.
\end{align*}
Finally, notice that the only incorrect similarity/dissimilarity predictions of $\SCes$ with respect to $\SCtcs$ are those between $P_j^+$ and $P_j^-$, for every $j<j^o$, and those between $X$ and $Y$, which gives us 
$2|P_j^+|\cdot|P_j^-| = d_{j}^2/2$ incorrect predictions for every $j<j^o$, and an additional 
$2|X|\cdot|Y| = 2c(d_{j^o}-c) \leq 2cd_{j^o} \leq \omega$ 
incorrect predictions between $X$ and $Y$. This implies that\
\[
\SCsmt{\SCes}{\SCtcs} 
\leq
\omega + \sum_{j=1}^{j^o-1}{d_{j}^2}/{8}
\] 
which is in turn bounded from above by $\sigma$. 
This completes the proof for the second case.
\end{proof}

The following simple lemma is of preliminary importance for the proof of Theorem \ref{th:batch}.
\begin{lemma}\label{l:ER}
Let $H = (\VV,E)$ be an Erdos-Renyi $G(n,p)$ graph. 
For each subgraph $H'(V',E') \subseteq H$ with 
$n'=|V'|$ nodes, when $p=\frac{\lambda \log n'}{n'}$ the following {\em separation property} holds: As $n'$ approaches infinity, the expected number $z$ of isolated vertices in $G'$ equals $(n')^{1-\lambda}$. Furthermore, in the special case when $n'=\frac{1}{p}$, we always have $z \ge \frac{1}{pe}$.
\end{lemma}
\begin{proof}
In order to prove these properties, it suffices to observe that, given any node in $V'$, the probability that it is isolated in $G'$ is equal to $(1-p)^{n'-1}$, which in turn is equal to  
$e^{-\lambda\log n'}=(n')^{-\lambda}$ as $n'$
approaches infinity. Hence we have $z=(n')^{1-\lambda}$.
By a similar argument, it is immediate to verify that in the case when $n'=\frac{1}{p}$ we have $z=\frac{1}{p}(1-p)^{\frac{1}{p}-1}$ which is never smaller than $\frac{1}{e p}$.
\end{proof}

\subsection{Proof of Theorem \ref{th:batch}}

\begin{proof}
Let $G' = (\VV,E')$ denote the undirected graph whose edge set $E'$ is made up of all pairs of vertices drawn in $S$. Since $S$ is drawn uniformly at random, $G'$ turns out to be an Erdos-Renyi graph $G'(n,p)$.


Setting $\lambda=2$ in Lemma~\ref{l:ER}, we have that for all clusters $C \in \scC$ such that $\frac{2\log |C|}{|C|} \le p$, cluster $C$ can be {\em completely} detected by \bnc\ (line 3 in Algorithm \ref{alg:lbanca}), with probability at least $\frac{1}{|C|}$.
Hence, the expected number of misclassification errors made when detecting such clusters is upper bounded by $1$ per cluster. In order to satisfy the assumption 
$\frac{2\log |C|}{|C|} \le p$, the size of these clusters 
must be equal to a value 
$\tau=\Omega\left(\rho\log\rho\right)$, where we set $\rho=\frac{1}{p}$.

%
Finally, we can conclude the proof observing that the total number of misclassification errors is bounded in expectation by the sum of the following two quantities: (i) the number of clusters larger than $\tau$, which in turn is bounded by $k$, and (ii) the total number of nodes belonging to the clusters smaller or equal to $\tau$, which in turn is bounded by $k\tau$:
\begin{equation}\label{e:firsteq}
\E\left[\SCer{\scC_{\mathrm{final}}}{Y}\right] = \scO\left(k(1+\tau)\right)=\scO\left(k\rho\log\rho\right)\,,
\end{equation}
thereby concluding the proof.
\end{proof}

\subsection{Proof of Theorem \ref{th:batchlb}}

\begin{proof}


As in the proof of Theorem \ref{th:batch}, we denote by $G' = (\VV,E')$ the undirected graph whose edge set $E'$ is made up of all pairs of vertices drawn in the training set $S$. Since $S$ is drawn uniformly at random, $G'$ turns out to be an Erdos-Renyi graph.


The basic idea in this proof is to construct a collection $\scH$ of $z$ disjoint subsets of $\VV$, call them $H_1, H_2 \ldots, H_z$, and, for all $j \in \{1,\ldots,z\}$, to {\em randomly} label all nodes of each subset $H_j$ using only a pair of classes of $\{1,\ldots,k\}$. These $z$ pairs of classes must be distinct and disjoint. The random labeling is accomplished in such a way that no algorithm can exploit the training set nor the information carried by $G$ to guess how each $H_j$ is labeled, while we always guarantee $\Phi^R_G(Y)\leq b$. More specifically, $\scH$ is created so as to satisfy the following three properties: 

\textbf{Property (i)} For all $j = 1,\ldots,z$, no pair of nodes in $H_j$ are connected by an edge in the training graph representation $G'$, i.e. for each pair of nodes $u,v \in H_j$, we have $(u,v) \not\in \PTr$.

\textbf{Property (ii)} Consider {\em any} possible vertex labeling from $\{1,\ldots,k\}$ of all sets in $\scH$ that uses at most $k-1$ classes, say $1,\ldots,k-1$.
Then if we assign  label $k$ (which is never used for vertices in the sets of $H_j$, to all the remaining nodes in $\VV$, i.e. those in $V\setminus \cup_{j=1}^z H_j$, we can {\em always} ensure 
that $\Phi^R_B \leq b$. 

\textbf{Property (iii)} For all $j = 1,\ldots, z$, we have that the expected size of $H_j$ (over the random draw of the training set $S$) is larger than $\frac{n^2}{2m\,e}$, if $m > \frac{n}{2}$, where $e$ is the base of natural logarithms, while it is $\Theta(b)$ if $m \le \frac{n}{2}$.


Figure~\ref{f:lb} provides a pictorial explanation of the  randomized labeling strategy we are going to describe.


\begin{figure}[t]
\begin{picture}(00,70)(-50,110)
\scalebox{0.5}{\includegraphics{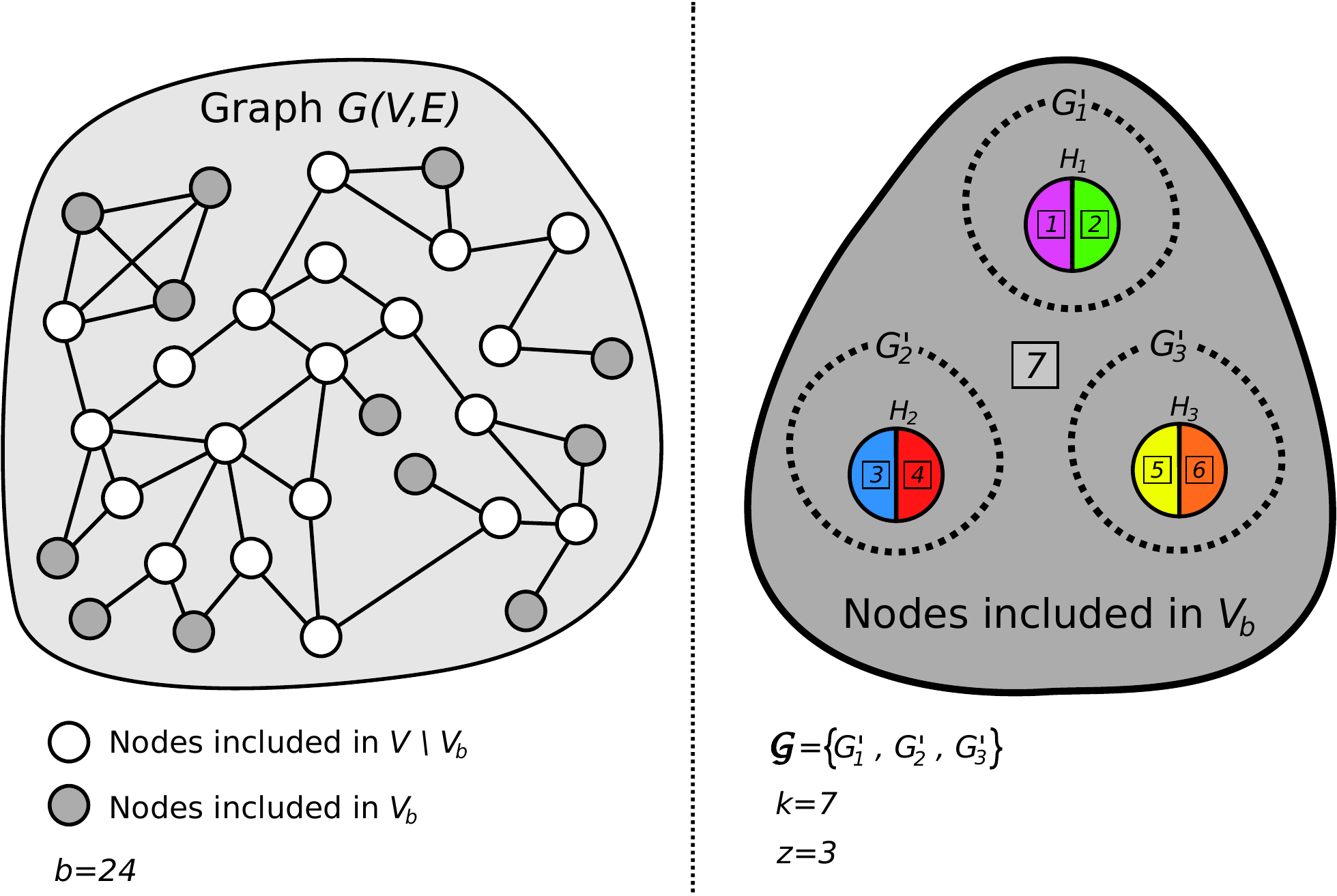}}
\end{picture}
\vspace{1.5in}
\caption{\label{f:lb}Illustration of the randomized labeling that achieves the lower bound in Theorem \ref{th:batchlb}.
\textbf{Left:} The side information graph $G=(\VV,E)$. Here $b=24$. $V_b$ is therefore made up of all $\frac{24}{2}=12$ grey nodes in the picture, which are the $12$ nodes 
with smallest $r(v)$ values among all nodes $v \in V$.
Since $R(12)<24$, if all white nodes are assigned to the same class we can ensure $\Phi^R_G < b = 2 \cdot 12 = 24$, independent of the chosen labels for the grey nodes.
\textbf{Right:} The grey area includes all the nodes of $V_b$.
For the depicted graph, we have $z=\left\lfloor \frac{2m}{n^2}\,|V_b|\right\rfloor=3$ and $k=7$.
In this case we thus have $\left\lfloor\ \frac{k-1}{2} \right\rfloor=3$. 
$\scG$ is the collection of the $3$ vertex-disjoint subgraphs 
$G'_1$, $G'_2$ and $G'_3$. 
The node set size of each of these subgraphs is equal to $\left\lfloor\frac{n^2}{2m}\right\rfloor$. 
The subsets of isolated vertices in these $3$ subgraphs are $H_1$, $H_2$, and $H_3$, 
which are depicted in this figure by the bicoloured circles. Each color represents a class.
For each $j$, the expected size of $H_j$ must be linear in the size of the node set of $G'_j$.
For $j = 1, 2, 3$, set $H_j$ is labeled by selecting uniformly at random a class between the two classes (or colors) $2j-1$ and $2j$.
All the remaining nodes in the grey area of this picture (together with the white nodes in the picture on the left) are given the same class $7$. 
Hence, for each pair of nodes $u$ and $v$ both belonging to $H_j$ for some $j$, we must have $(u,v) \not\in \PTr$.
On the contrary, for each pair $u$ and $v$ with $u \in H_j$, for some $j$, and $v \not\in H_j$, 
we must have $y_{u,v}=0$. Neither the information of the training set nor the graph topology of $G$ can be used to predict how the nodes in $H_1$, $H_2$, and $H_3$, are labeled. In fact, {\em any} algorithm will make $\frac{1}{2}$ mistakes in expectation over the randomized 
labeling on these nodes. On the other hand, it holds by construction that $\Phi^R_G \leq b$.
}

\end{figure}


We now describe in detail the randomized labeling strategy (a randomized similarity matrix $Y$ representing a clustering with $k$ clusters), and derive a lower bound for $\E_{Y}[\SCer{\CC}{Y}]$ when $\scH$ satisfies all of the above properties. 

Let $z \leq \lfloor\frac{k-1}{2}\rfloor$. Once we constructed such a collection $\scH$ of clusters, we associate a distinct pair of classes in $\{1,\ldots,k\}$ with each $H_j$ in such a way that all these class pairs are distinct and disjoint. This allows us to always leave one class out (say, class $k$) for labeling all remaining vertices in $\VV$. In particular,
we associate with $H_j$ with the class pair $(2j-1, 2j)$, and then adopt the following randomized strategy:

For all $j= 1,\ldots,z$, set $H_j$ is split uniformly at random into two subsets $H'_{j}$ and $H''_{j}$, and we label $H'_j$ by class $2j-1$ and $H''_i$ by class $2j$. All remaining nodes in
$\VV \setminus \cup_{j=1}^z H_j$ are labeled with class $k$. 

This randomized labeling strategy ensures that, in order to guess the true clustering $Y$, no learning algorithm can exploit the information provided by $S$, since for all node pairs $(v,w)$ with $v \in H_j$, for some $j \in \{1,\ldots, z\}$, one of two cases hold:

\textit{Case (a):} $w \in H_j$, which implies that $(v,w) \not\in \PTr$, because of Property (i). We have therefore no training set information related to the similarity of nodes laying in the same set $H_j$.

\textit{Case (b):} $w \notin H_j$. In this case, whenever $(v,w) \in \PTr$, we {\em always} have $y_{v,w}=0$, and this information cannot be exploited to guess the randomized labeling of $H_j$.

In short, no training information can be exploited to guess how each set $H_j$ is split into the two subsets $H'_{j}$ and $H''_{j}$. Furthermore, the side information graph $G$ cannot be exploited because the randomized labeling is completely independent of the $G$'s topology, while Property (ii) ensures that $\scH$ is selected in such a way that we always have $\Phi^R_G \leq b$. This entails that any clustering algorithm will incur an expected number of misclassification errors proportional to 
\[
\sum_{j=1}^z |H_j| = \Omega\left(\min\left\{\frac{n^2}{m}\,z, n\right\}\right),
\] 
the latter equality deriving from Property (iii).


We now turn to describing the detailed construction of $\scH$. We first need to find 
$V_{b} \subset \VV$ such that
%
$|V_b|=\left\lfloor\frac{b}{2}\right\rfloor$, and for any labeling of $\VV$ such that all nodes contained in $\VV \setminus V_b$ are uniformly assigned to the same class, we always have $\Phi^R_G \leq b$. 
All sets $H_1,\ldots, H_z$ are subsets of $V_b$, so as to satisfy Property (ii). Note that the assumption $b \ge 4$ is used here to ensure $|V_b| > 1$.
Thereafter, we will explain how to select the $z$ subsets satisfying Property (i), and show that their size is bounded from below as required by Property (iii). This will lead to the claimed lower bound.

\textbf{Satisfaction of property (ii).} Let $r_{v,w}$ be the effective resistance between nodes $v$ and $w$ in $G = (\VV,E)$, and $r(v) = \sum_{w\,:\, (v,w) \in E} r_{v,w}$. 
Moreover, given any integer $h \le n$, let 
$R(h) = \min_{\{v_1, v_2, \ldots v_h\} \subseteq V} \sum_{\ell = 1}^h r(v_{\ell})$. We have
\[
R(h) 
\le 
\frac{h}{n}\,\sum_{v \in \VV} r(v)
=
\frac{2h}{n}\,\sum_{(v,w) \in E} r_{v,w}
=
\frac{2h(n-1)}{n} 
< 2h\,.
\]
%
%

%
Let now $V_b$ be the subset of $\VV$ containing the $\left\lfloor\frac{b}{2}\right\rfloor$ nodes $v_1, v_2, \ldots v_{\left\lfloor b/2\right\rfloor}$ achieving the smallest values of $r(v)$ over all $v \in \VV$. Hence we must have $R\left\lfloor\frac{b}{2}\right\rfloor \leq b$, which implies that for {\em any} vertex labeling such that all nodes in $\VV \setminus V_b$ are labeled with the same class, we have $\Phi^R_G \leq b$. As anticipated, we will construct $\scH$ using only subsets from $V_b$, and we assign to all nodes in $V \setminus V_b$ the same class, thereby fulfilling Property (ii).





\textbf{Definition of $z$.} Let 
\[
z = \min\left\{f(b,n,m), \left\lfloor\frac{k-1}{2}\right\rfloor\right\},\ \quad\ {\mbox{where}}\quad
f(b,n,m)=\max\left\{\left\lfloor\frac{b\,m}{n^2}\right\rfloor, 1\right\}\,.
\] 

\textbf{Satisfaction of Property (i).}\\ 
Let $\scH'$ be a collection of disjoint subsets of $V_b$ created as follows. If $\left\lfloor\frac{b}{2}\right\rfloor<\left\lfloor\frac{n^2}{2m}\right\rfloor$, then $f(b,n,m)$ and $z$ are both equal to $1$ and $\scH'$ contains only $V_b$. In all other cases, $\scH'$ is generated by selecting uniformly at random $z$ disjoint subsets of $V_b$ such that each node subset contains 
$\left\lfloor\frac{n^2}{2m}\right\rfloor$ nodes (observe that even in the latter case we may have $z=1$ when $k=3$
or $1 \le \left\lfloor\frac{bm}{n^2}\right\rfloor<2$).
The collection of subsets $\scH = \{H_1,\ldots,H_z\}$ 
is constructed as described next.
%
%
%
Let $\scG \equiv \{G'_1, G'_2, \ldots, G'_{z}\}$, where $G'_j$ is the subgraph of $G'$ induced by the nodes in the $j$-th set of $\scH'$. We create $z$-many disjoint subsets $H_1, H_2, \ldots, H_{z}$ by selecting all vertices that are {\em isolated} in each graph of $\scG$, and set 
$\scH\equiv\{H_1, H_2, \ldots, H_z\}$. Property (i) is therefore satisfied. 

\textbf{Satisfaction of property (iii).} 
By definition of $\scH'$, each graph of the collection $\scG$ has $\left\lfloor\frac{n^2}{2m}\right\rfloor$ nodes.
Using the second part of Lemma~\ref{l:ER}, we conclude that 
the expected size of each set in $\scH$ is not smaller than 
$\frac{\left\lfloor n^2/2m \right\rfloor}{e}$.

Now, if $\left\lfloor\frac{b}{2}\right\rfloor\ge\left\lfloor\frac{n^2}{2m}\right\rfloor$, because the set of vertices of each graph in $\scG$ is made up of 
$\left\lfloor\frac{n^2}{2m}\right\rfloor$ nodes, we can conclude that the expected number of isolated nodes contained in each graph of the collection $\scG$ is linear in its vertex set size. If instead we have $\left\lfloor\frac{b}{2}\right\rfloor<\left\lfloor\frac{n^2}{2m}\right\rfloor$, by the way $f(b,n,m)$ is defined and because of property (iii), 
the number of isolated nodes contained in $V_b$ is linear in the size of $V_b$ itself.
Furthermore, the size of the (unique) set of nodes contained in $\scH$ is
linear in $b$, as well as the expected number of mistakes that can be forced.
In fact, the claimed lower bound is $\Omega(b)$ when $\left\lfloor\frac{b}{2}\right\rfloor<\left\lfloor\frac{n^2}{2m}\right\rfloor$.

Hence the collection of sets $\scS$ so generated fulfils at the same time Properties (i), (ii) and (iii).

In order to conclude the proof, we compute our lower bound based on the definition of $z$ and Property (iii).
As anticipated, because of the randomized labeling strategy, the expected number of misclassification errors made by any algorithm is proportional to 
$\sum_{j=1}^z |H_j| = \Omega\left(\min\left\{\frac{n^2}{m}z, n\right\}\right)$. Plugging in the values of $z$ yields
\[
\sum_{j=1}^z |H_j| = \Omega\left(\min\left\{\frac{\min\left\{\max\left\{\left\lfloor\frac{b\,m}{n^2}\right\rfloor, 1\right\}, \left\lfloor\frac{k-1}{2}\right\rfloor\right\}}{m/n^2}, n\right\}\right)
=
\Omega\left(\min\left\{\frac{n^2}{m}\,k, b\right\}\right)\,,
\]
and the proof is concluded.




\end{proof}

\end{document}